\documentclass{article} 
\usepackage{iclr2026_conference,times}

\usepackage[utf8]{inputenc} 
\usepackage[T1]{fontenc}    
\usepackage[backref=page]{hyperref}       
\usepackage{url}            
\usepackage{booktabs}       
\usepackage{amsfonts}       
\usepackage{nicefrac}       
\usepackage{microtype}      
\usepackage{xcolor}         
\usepackage{amsmath}
\usepackage{amsthm}
\usepackage{amssymb}
\usepackage{mathtools}
\usepackage{mathrsfs}
\usepackage{graphicx}
\usepackage{wrapfig}
\usepackage{tikz}
\usetikzlibrary{positioning,arrows.meta}
\usepackage{caption}
\usepackage{float}
\usepackage{subcaption}
\usepackage{geometry}
\usepackage{multirow}
\usepackage{adjustbox}
\usepackage{enumitem}

\usepackage{algorithm}
\usepackage{algorithmic}
\usepackage{listings}

\usepackage{tcolorbox}
\tcbset{
  algorithmbox/.style={
    sharp corners,
    boxrule=0.8pt,
    left=2mm, right=2mm, top=1mm, bottom=1mm,
    fonttitle=\bfseries,
    colback=#1!10!white,
    colframe=#1!70!black
  }
}

\def\eqref#1{equation~\ref{#1}}

\def\1{\bm{1}}

\DeclareMathAlphabet{\mathsfit}{\encodingdefault}{\sfdefault}{m}{sl}
\SetMathAlphabet{\mathsfit}{bold}{\encodingdefault}{\sfdefault}{bx}{n}

\newcommand{\E}{\mathbb{E}}

\DeclareMathOperator*{\argmax}{arg\,max}
\DeclareMathOperator*{\argmin}{arg\,min}

\newtheorem{theorem}{Theorem}[section]
\newtheorem{lemma}{Lemma}[section]
\newtheorem{corollary}{Corollary}[section]
\newtheorem{proposition}{Proposition}[section]

\definecolor{simHigh}{HTML}{FFFFFF} 
\definecolor{simMed}{HTML}{B4D7FF}  
\definecolor{simLow}{HTML}{79B8FF}  
\definecolor{simMin}{HTML}{377FF2}  

\title{Evaluation-Aware Reinforcement Learning}

\usepackage{authblk}

\author{Shripad Vilasrao Deshmukh \quad Will Schwarzer \quad Scott Niekum}
\affil{University of Massachusetts Amherst \\ \texttt{\{svdeshmukh, wschwarzer, sniekum\}@cs.umass.edu}}

\iclrfinalcopy 
\begin{document}

\maketitle


\begin{abstract}
    Policy evaluation is a core component of many reinforcement learning (RL) algorithms and a critical tool for ensuring safe deployment of RL policies. However, existing policy evaluation methods often suffer from high variance or bias. To address these issues, we introduce Evaluation-Aware Reinforcement Learning (EvA-RL), a general policy learning framework that considers evaluation accuracy at train-time, as opposed to standard post-hoc policy evaluation methods. Specifically, EvA-RL directly optimizes policies for efficient and accurate evaluation, in addition to being performant. We provide an instantiation of EvA-RL and demonstrate through a combination of theoretical analysis and empirical results that EvA-RL effectively trades off between evaluation accuracy and expected return. Finally, we show that the evaluation-aware policy and the evaluation mechanism itself can be co-learned to mitigate this tradeoff, providing the evaluation benefits without significantly sacrificing policy performance. This work opens a new line of research that elevates reliable evaluation to a first-class principle in reinforcement learning.
\end{abstract}



\section{Introduction}
\label{sec:intro}

Recent years have seen a surge in the application of reinforcement learning (RL)~\citep{SuttonBarto1998} to real-world tasks, including industrial control~\citep{mirhoseini2020chip}, robotics~\citep{Kalashnikov2018QTOpt}, and healthcare~\citep{Fox2020ClosedLoopBG}. In such settings, reliable policy evaluation---accurately estimating the expected performance of a policy across states before deployment---is essential for safe and informed decision-making~\citep{DulacArnold2019ChallengesRealWorldRL}. Beyond its role in deployment decisions, policy evaluation is also a central component of many policy learning algorithms~\citep{SuttonBarto1998,hcpi}, making it an important subtopic of RL research in its own right.

Existing policy evaluation methods fall into two broad classes: on-policy and off-policy~\citep{ThomasTheocharousGhavamzadeh2015HCOPE}. On-policy estimation requires direct interaction with the environment, which can be prohibitively costly. Off-policy evaluation (OPE) methods sidestep this cost by leveraging data collected under policies that can be different from the evaluation policy. However, OPE methods face their own challenges: importance-sampling estimators~\citep{PrecupSuttonDasgupta2001OffPolicyTD} typically suffer from high variance as the task horizon grows~\citep{Liu2018CurseOfHorizon}, while direct methods~\citep{fonteneau2013batch} often incur high bias when the off-policy data provides insufficient coverage of the states and actions visited by the evaluation policy. In practice, limited data, poor coverage, and inaccurate environment models frequently lead to poor evaluation.

To address these challenges, we propose to incorporate evaluation accuracy as a training-time objective. Rather than treating evaluation as a post-hoc step applied after a policy has been learned, we instead learn policies that not only maximize performance, but also enable accurate evaluation under a given evaluation scheme. We call this approach \textit{Evaluation-Aware Reinforcement Learning} (EvA-RL). The central premise is that by optimizing for \textit{evaluability}---the property of incurring low evaluation error under a specified estimator---during training, the expected gap between a policy's true and estimated performance is explicitly minimized, rather than left to chance.

We first formalize the general EvA-RL objective, which balances performance maximization with evaluation error minimization under a given evaluation scheme. To enable tractable train-time evaluation, we instantiate EvA-RL with an evaluation scheme that estimates a policy's state-values conditioned on its behavior in a small set of scenarios. This is akin to a driver's test that gauges general driving ability from performance on a small set of scenarios. We formally refer to the environment where such an assessment takes place as an `assessment environment', designed to elicit value-informative behavior efficiently. We employ a parameterized value evaluator to predict the state-values of a policy conditioned on its assessment behavior.

Our theoretical analysis reveals that, when the evaluation scheme is held fixed, a fundamental tradeoff can emerge between policy performance and evaluation accuracy---enforcing tighter evaluation accuracy leads to a final return that is at best equal to, and in general lower than, that of a policy trained with a more permissive evaluation-error tolerance. We then show that this tradeoff can be mitigated by co-learning the value evaluator alongside the policy, partially shifting the burden of evaluation error minimization onto the evaluator and easing the constraints on policy optimization. Experiments across diverse discrete and continuous control domains demonstrate that EvA-RL consistently reduces evaluation error while maintaining competitive returns.

Taken together, this work opens a new line of RL research that treats evaluation accuracy as a first-class objective alongside performance maximization, enabling the development of policies that are both performant and efficiently evaluable.
\section{Background and Related Work}

\noindent\textbf{Markov decision process (MDP).} 
We work in the setting of a finite horizon MDP~\citep{puterman2014markov}, defined by a tuple $(\mathcal{S}, \mathcal{A}, p, r, \gamma, \mu)$. Here, $\mathcal{S}$ denotes the set of all possible states and $\mathcal{A}$ denotes the set of all available actions. The transition function $p: \mathcal{S} \times \mathcal{A} \xrightarrow{} \Delta(\mathcal{S})$ defines the distribution over next states given a state-action pair, and the reward function $r: \mathcal{S} \times \mathcal{A} \xrightarrow{} \mathbb{R}$ defines the expected reward $r(s, a)$ for taking action $a$ in state $s$. The discount factor is $\gamma \in [0, 1)$, and $\mu \in \Delta(\mathcal{S})$ defines the initial state distribution. A trajectory $H$ of horizon $T < \infty$ is defined as $(S_0, A_0, R_0, S_1, A_1, R_1, \dots, S_T)$, with discounted return \mbox{$G(H) = \sum_{t=0}^{T-1}\gamma^t r(S_t, A_t)$}. A stochastic policy $\pi: \mathcal{S} \xrightarrow{} \Delta(\mathcal{A})$ defines a distribution over actions conditioned on the current state. The state-value function \mbox{$V_\pi(s) = \mathbb{E}_{H \sim \pi}[G(H)| S_0 = s]$} represents the expected return under $\pi$ starting at state $s$, while the state-action value function \mbox{$Q_\pi(s,a) = \mathbb{E}_{H \sim \pi}[G(H)| S_0 = s, A_0 = a]$} conditions additionally on the first action. The aggregate performance of policy $\pi$ is \mbox{$J_\pi = \mathbb{E}_{s \sim \mu}V_\pi(s)$}.

\noindent\textbf{Policy evaluation.} In this work, the term policy evaluation is used to indicate the estimation of the state-value function of the policy under evaluation. The term is also sometimes used to refer to estimating the aggregate performance $J_\pi$; given an estimate of $V_\pi$, this can be obtained by averaging over states sampled from the initial state distribution $\mu$. The simplest approach to policy evaluation is Monte Carlo (MC) estimation, which draws $N$ trajectory samples under $\pi$ starting at $s$ and averages their returns: \mbox{$V_\pi^{\text{MC}}(s) = \frac{1}{N}\sum_{i=1}^{N}G(H_i | S_0 = s)$}. While unbiased, producing reliable MC estimates can be prohibitively costly due to the required environment interactions.

\noindent\textbf{Off-policy evaluation (OPE).} OPE methods address this cost by reusing data collected under a different \emph{behavior policy} $\pi_b$, avoiding the need to deploy $\pi$ directly~\citep{Uehara2022OPEReview}. The most straightforward OPE approach is trajectory importance sampling (TIS), which reweighs each trajectory return $G(H)$ by the importance ratio \mbox{$w(H) = \prod_{t=0}^{T-1}\frac{\pi(A_t \mid S_t)}{\pi_b(A_t \mid S_t)}$}, yielding: \mbox{$V_\pi^{\text{TIS}}(s) = \frac{1}{N}\sum_{i=1}^{N}w(H_i)G(H_i | S_0 = s)$}. However, TIS suffers from potentially exponential variance in the importance weights. Several estimators have been proposed to address this. Per-decision importance sampling (PDIS)~\citep{PrecupSuttonDasgupta2001OffPolicyTD} applies importance weights incrementally at each time step, generally reducing variance relative to TIS but still struggling in long-horizon tasks. Fitted Q-evaluation (FQE)~\citep{fonteneau2013batch,le2019batch}, also known as `direct method', avoids importance weighting entirely by learning a Q-function via Bellman consistency on off-policy data, but can suffer from high bias under poor state-action coverage. Doubly robust (DR) estimators~\citep{JiangLi2016DoublyRobustRL, ThomasBrunskill2016DataEfficientOPE} combine importance sampling with learned value functions, reducing variance while remaining consistent when either component is accurate.

In our empirical analyses, we compare policies trained traditionally and evaluated using the above OPE methods (TIS, PDIS, FQE, and DR) against policies trained under our proposed evaluation-aware framework.

\noindent\textbf{Behavior policy search.}
A complementary line of work by~\citet{Hanna2017BehaviorPolicySearch} optimizes the data collection strategy: given an evaluation policy $\pi$, they search for a behavior policy $\pi_b$ that, when deployed to gather new trajectories, minimizes the mean-squared error of the resulting IS estimates. Our work addresses the converse problem -- rather than optimizing how data is collected, we search for an evaluation policy $\pi$ that is both performant and accurately evaluable from existing assessment data.

\section{EvA-RL: Evaluation-Aware Reinforcement Learning}\label{sec:evarl_method}

In this section, we formally introduce evaluation-aware reinforcement learning (EvA-RL) and present a practical approach for evaluation-aware learning that achieves low evaluation error and high aggregate performance.


\noindent{\textbf{Optimization problem:}} Consider an evaluation scheme that, given a state, a policy, and data $\mathcal{D}$ collected from the environment, outputs a real-valued estimate $\hat{V}(s;\, \pi,\, \mathcal{D})$ of the policy's state-value. The quality of this evaluation scheme for a given policy $\pi$ is measured by the mean squared error (MSE): $\mathcal{L}_\mathrm{MSE}(\pi, \hat{V}) = \mathbb{E}_{s \sim \mu}\Big[\big(V_\pi(s) - \hat{V}(s;\, \pi,\, \mathcal{D})\big)^2\Big]$. We refer to the property of a policy to incur low MSE w.r.t.\ a given evaluation scheme as its \textit{evaluability}. We formulate evaluation-aware policy learning as the following optimization problem:
\begin{equation}
    \label{eq:value_predictability_objective_general}
    \mathrm{arg}\max_{\pi} \; \Bigg[\underbrace{\mathbb{E}_{s \sim \mu}\big[\mathbb{E}_{H \sim \pi}[G(H) \mid S_0 = s]\big]}_{J_\pi} - \beta\; \underbrace{\mathbb{E}_{s \sim \mu}\Big[\big(\mathbb{E}_{H \sim \pi}[G(H) \mid S_0 = s] - \hat{V}(s;\, \pi,\, \mathcal{D})\big)^2\Big]}_{\mathcal{L}_\mathrm{MSE}(\pi,\, \hat{V})}\Bigg]
\end{equation}
where $\beta > 0$ is the \textit{evaluability coefficient} that controls the tradeoff between improving aggregate performance and improving evaluability. 

\noindent\textbf{Practical challenge:} We consider an on-policy approach to optimizing the EvA-RL objective (Eq.~\ref{eq:value_predictability_objective_general}), which requires evaluating each candidate policy during training. Since the evaluation scheme must be queried repeatedly as we search for a policy that improves both performance and evaluability, computational and statistical efficiency is critical. This rules out standard OPE techniques: importance sampling methods require processing large amounts of off-policy data, and direct methods such as FQE~\citep{le2019batch} require solving for $Q_\pi^{\text{FQE}}$ separately for every candidate policy $\pi$. Moreover, while Monte Carlo return samples are naturally collected at each on-policy update and provide unbiased estimates of state-values, they are high-variance. Instead, our approach is to gather a small amount of policy behavior in a low-cost \textit{assessment environment} and use it to predict the state-values. Concretely, we roll out the candidate policy $\pi$ in the assessment environment to collect trajectories, and then produce state-value estimates by conditioning on the observed behavior. This makes the auxiliary data $\mathcal{D}$ used by our estimator to compute $\hat{V}(s; \pi, \mathcal{D})$ cheap to collect for any candidate policy, enabling efficient in-the-loop evaluation.

\noindent\textbf{Assessment environment:} For the rest of the paper, we make a distinction: the \textit{deployment environment} refers to the environment where the policy would ultimately be deployed, and the \textit{assessment environment} refers to an environment in which value-informative behavior can be observed more easily, cheaply, or safely than in the deployment environment.

The assessment environment can relate to the deployment environment in various ways along a spectrum. At one extreme, it may be the deployment environment itself but restricted to a curated set of initial states or shorter time horizons. It may also share the same state-action spaces but simplify the dynamics, for example by reducing stochasticity or the number of interacting agents. At the other extreme, it may be an entirely separate environment, such as a simulation, that approximates the deployment setting. Such assessment environments naturally emerge in many real-world control tasks---simulators such as CARLA~\citep{Dosovitskiy2017CARLA} for autonomous driving and ODE-based models~\citep{man2014uva} for diabetes management are routinely used during development and can serve this role. Additionally, a domain expert can design a curated set of scenarios within the deployment environment to probe agent behavior, and these can serve as an assessment environment.


\noindent{\textbf{Formalism:}} We now formalize the setup introduced above. In the MDP notation defined earlier, the deployment environment is written as $\mathcal{M}_D = (\mathcal{S}_D, \mathcal{A}_D, p_D, r_D, \gamma_D, \mu_D)$ and the assessment environment as $\mathcal{M}_A = (\mathcal{S}_A, \mathcal{A}_A, p_A, r_A, \gamma_A, \mu_A)$. We assume that compatibility conditions are satisfied so that the policy $\pi$ can operate in both $\mathcal{M}_D$ and $\mathcal{M}_A$ (see Appendix~\ref{sec:compatibility_conditions}). As a distinguishing feature, we consider the assessment environment to have $k$ distinct start states, allowing us to gather behavioral information across $k$ distinct scenarios. Formally, $\mu_A(s) = 1/k$ if $s \in \{s^1, s^2, \dots, s^k\} \subset \mathcal{S}_A$ and $0$ otherwise. We denote the state-value function of $\pi$ in the deployment environment by $V_D^\pi$ and in the assessment environment by $V_A^\pi$. We denote by $\Xi_A$ the behavioral information gathered by executing $\pi$ in $\mathcal{M}_A$ across the $k$ assessment scenarios. This may include full trajectories or summary statistics of
the policy's behavior, such as state-values or risk measures. Our goal is to estimate $V_D^\pi$ by conditioning on $\Xi_A$. We denote this estimate for a deployment state $s \in \mathcal{S}_D$ by $\hat{V}_D^\pi(s \mid \Xi_A)$. This is a concrete instantiation of the general estimator $\hat{V}(s;\, \pi,\, \mathcal{D})$ introduced earlier, where the auxiliary data $\mathcal{D}$ is the assessment behavior $\Xi_A$.


Substituting the assessment-conditioned evaluation scheme into the general objective (Eq.~\ref{eq:value_predictability_objective_general}) and making the deployment environment explicit, the EvA-RL optimization problem becomes:
\begin{equation}
    \label{eq:value_predictability_objective}
    \mathrm{arg}\max_{\pi} \; \Bigg[\underbrace{\mathbb{E}_{s \sim \mu_D;\, H \sim \pi \mid s,\, \mathcal{M}_D}[G(H)]}_{\text{performance}} - \beta\; \underbrace{\mathbb{E}_{s \sim \mu_D}\Big[\big(\mathbb{E}_{H \sim \pi\mid \mathcal{M}_D}[G(H) \mid S_0 = s] -  \hat{V}_D^{\pi}(s \mid \Xi_A)\big)^2\Big]}_{\text{evaluability}}\Bigg]
\end{equation}
While the above objective can in principle be optimized using various approaches (e.g., value-based methods), in this paper we adopt a policy gradient approach and approximate $V_D^{\pi}(s) = \mathbb{E}_{H \sim \pi\mid \mathcal{M}_D}[G(H) \mid S_0 = s]$ using Monte Carlo returns from trajectories sampled in $\mathcal{M}_D$, following common practice in methods such as A2C~\citep{KondaTsitsiklis1999ActorCritic} and TRPO~\citep{Schulman2015TRPO}.

\subsection{Performance-evaluability tradeoff with a fixed evaluator}
\label{sec:theoretical_insights}

The EvA-RL objective (Eq.~\ref{eq:value_predictability_objective}) jointly optimizes for performance and evaluability. A natural question is whether both objectives can be improved simultaneously. We analyze this question in the setting where the evaluation scheme is held fixed throughout policy optimization. The following result shows that, in this setting, the expected return under stronger evaluability enforcement is non-increasing.

\begin{proposition}\label{thm:monotonicity_of_pred_gap_returns}
    Increasing $\beta$ in the EvA-RL objective (Eq.~\ref{eq:value_predictability_objective}) with a fixed evaluation scheme leads to non-increasing evaluation error $\mathcal{L}_\mathrm{MSE}$ and non-increasing expected return $J_\pi$ of the resulting policy (proof in Appendix~\ref{sec:monotonicity_pred_gap_returns}).
\end{proposition}

This result requires only the structure of maximizing $f(\pi) - \beta\, g(\pi)$ over policies and therefore applies to any fixed evaluation scheme. It establishes that  enforcing tighter evaluability leads to an expected return that is at best equal to, and in general lower than, that of a policy trained with a more permissive evaluation-error tolerance. This motivates co-learning the evaluation scheme alongside the policy, as we propose in the next section.

\subsection{EvA-RL with a co-learned value evaluator}\label{sec:practical_evarl}






The tradeoff identified in Proposition~\ref{thm:monotonicity_of_pred_gap_returns} arises because the policy alone must bear the burden of evaluability under a fixed evaluation scheme. A natural remedy is to allow the value evaluator itself to adapt alongside the policy, so that the burden of reducing evaluation error is shared. We instantiate this idea by parameterizing the value evaluator as a transformer~\citep{Vaswani2017Attention} with parameters $\phi$ that takes as input a query state $s \in \mathcal{S}_D$ along with assessment behavior $\Xi_A$, and outputs an estimate $\hat{V}_D^{\pi, \phi}(s \mid \Xi_A)$ of $V_D^{\pi}(s)$. In practice, we found that conditioning on the assessment start-states and their corresponding Monte Carlo returns was sufficient, specializing $\Xi_A = \{(s^i, G(H^i)) \mid H^i \sim \pi;\; s^i;\; \mathcal{M}_A\}_{i=1}^{k}$, where $G(H^i)$ denotes the discounted return of trajectory $H^i$. Figure~\ref{fig:predictability_transformer} illustrates the architecture; implementation details are in Appendix~\ref{sec:appendix_pred_transformer}.

\begin{figure}[h]
    \centering
    \includegraphics[width=0.55\textwidth]{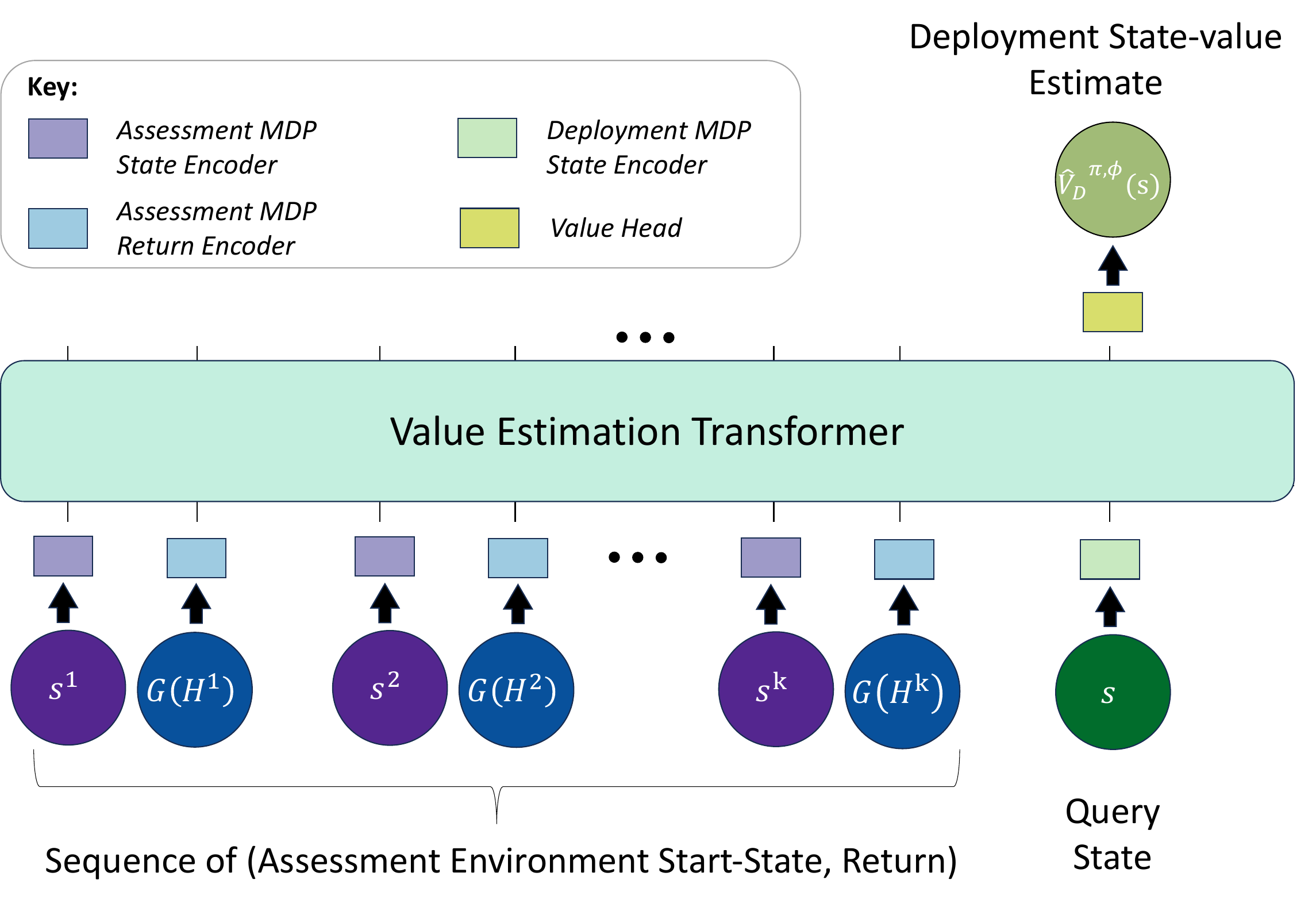}
    \caption{\textbf{Value estimation transformer.} The value evaluator $\hat{V}_D^{\pi, \phi}$ is a transformer encoder that takes as input the assessment start-states $\{s^1, s^2, \ldots, s^k\}$, the corresponding returns $\{G(H^1), G(H^2), \ldots, G(H^k)\}$ from rolling out policy $\pi$ in $\mathcal{M}_A$, and a query state $s \in \mathcal{S}_D$. It outputs an estimate of the query state's value $\hat{V}_D^{\pi, \phi}(s \mid \Xi_A)$ in the deployment environment. Implementation details are in Appendix~\ref{sec:appendix_pred_transformer}.}
    \label{fig:predictability_transformer}
\end{figure}

Since both the policy $\pi_\theta$ and the evaluator are initialized randomly, the evaluator cannot produce meaningful value estimates at the start of training. We therefore begin with a warm-up phase of $T_\mathrm{warm}$ updates in which the policy is trained using standard policy gradient without any evaluability penalty. During warm-up, we populate a replay buffer $\mathcal{B}$ with tuples $(s,\, G(H),\, \Xi_A)$, where $s \sim \mu_D$ is a deployment state, $G(H)$ is the Monte Carlo return of a trajectory starting at $s$ under the current policy in $\mathcal{M}_D$, and $\Xi_A$ is the assessment behavior of the same policy in $\mathcal{M}_A$. The collected data is then used to initialize the evaluator, after which evaluation-aware policy learning begins. As training proceeds, $\mathcal{B}$ is maintained with data from the $m$ most recent policies, providing the evaluator with a diverse set of policy behaviors to learn from.

At each subsequent iteration, the evaluator parameters $\phi$ are first updated by minimizing the mean squared error in the prediction of state-values for the policies in the buffer $\mathcal{B}$, and the policy parameters $\theta$ are then updated to optimize the performance and evaluability under the new evaluator:
\begin{align}
\phi_{n+1} &\;\leftarrow\; \argmin_{\phi} \; \E_{(s, G(H), \Xi_A) \sim \mathcal{B}} \left[ \bigl(G(H) - \hat{V}_D^{\pi_{\theta_n}, \phi_n}(s \mid \Xi_A)\bigr)^2 \right], \nonumber \\
\theta_{n+1} &\;\leftarrow\; \argmax_{\theta} \; \E_{s \sim \mu_D;\, H \sim \pi_{\theta_n} \mid s, \mathcal{M}_D;\, \Xi_A \sim \pi_{\theta_n} \mid \mathcal{M}_A} \left[ G(H) - \beta\,\bigl(G(H) - \hat{V}_D^{\pi_{\theta_n}, \phi_{n+1}}(s \mid \Xi_A)\bigr)^2 \right]. \nonumber
\end{align}

Algorithm~\ref{alg:eva_rl} (Appendix~\ref{sec:algo}) summarizes the full procedure.

The evaluator update is a straightforward regression of deployment returns conditioned on assessment behavior $\Xi_A$. The policy update is more nuanced, as changing the policy parameters $\theta$ affects both the deployment returns $G(H)$ and the assessment behavior $\Xi_A$. Writing $\E_{s \sim \mu_D;\, H \sim \pi_{\theta} \mid s, \mathcal{M}_D;\, \Xi_A \sim \pi_{\theta} \mid \mathcal{M}_A}$ compactly as $\E_{H, \Xi_A \sim \pi_{\theta}}$, the policy gradient decomposes as:
\begin{align}
    \nabla_{\theta} J(\theta)
    &= \nabla_{\theta}\;\E_{H, \Xi_A \sim \pi_{\theta}}
    \left[
    G(H)
    - \beta\,\bigl(G(H) - \hat{V}_D^{\pi_\theta, \phi}(s \mid \Xi_A)\bigr)^2
    \right] \\
    &= \underbrace{\nabla_{\theta}\E_{H, \Xi_A \sim \pi_{\theta}}\left[G(H)\right]}_{\text{standard policy gradient}}
    - 2\,\beta\,
    \E_{H, \Xi_A \sim \pi_{\theta}}\left[\bigl(
    G(H) - \hat{V}_D^{\pi_\theta, \phi}(s \mid \Xi_A)
    \bigr)
    \bigl(
    \nabla_{\theta}G(H)
    - \nabla_{\theta}\hat{V}_D^{\pi_\theta, \phi}(s \mid \Xi_A)
    \bigr)\right]. \nonumber
\end{align}
The first term is the standard policy gradient~\citep{Sutton1999PolicyGradient}. The second term is an evaluability correction that penalizes policy changes which increase the gap between true and estimated values. Crucially, the evaluator receives the policy's assessment returns $G(H^i)$ as input, so its gradient with respect to the policy parameters decomposes via the chain rule:
\mbox{$
\nabla_{\theta} \hat{V}_D^{\pi_\theta, \phi}(s \mid \Xi_A)
= \sum_{i=1}^k
\frac{\partial\,\hat{V}_D^{\pi_\theta, \phi}(s \mid \Xi_A)}{\partial\,G(H^i)}
\cdot \nabla_{\theta}\,G(H^i).
$}
Changing the policy changes the assessment returns, which in turn changes the evaluator's predictions. The policy is therefore directly incentivized to produce assessment behavior that is informative for value estimation---not just to perform well in deployment.
\section{Experiments and Results}\label{sec:experiments}

In this section, we present experiments designed to answer the following questions: (1)~When using a fixed value evaluator, does increasing the evaluability coefficient produce the performance--evaluability tradeoff predicted by Proposition~\ref{thm:monotonicity_of_pred_gap_returns}, and can co-learning mitigate it? (2)~Does the co-learned evaluator produce more accurate value estimates for EvA-RL policies than standard OPE methods? (3)~When comparing full pipelines end-to-end --- EvA-RL with its co-learned evaluator versus standard RL evaluated with OPE methods --- does EvA-RL achieve competitive returns while reducing evaluation error?

\noindent\textbf{Setup:} We evaluate on three discrete-action Gymnax environments (Asterix, Freeway, and Space Invaders) and three continuous-action Brax environments (HalfCheetah, Reacher, and Ant). Our implementation builds on PureJaxRL~\citep{lu2022discovered}. Agents are trained for 10M environment interactions, and results are averaged over 20 random seeds with standard error reported. We compare value estimates from the co-learned EvA-RL evaluator against four OPE baselines: fitted Q-evaluation (FQE), trajectory importance sampling (TIS), per-decision importance sampling (PDIS), and the doubly robust (DR) estimator. For FQE and DR, we adapt code from the Scope-RL library~\citep{kiyohara2023towards} and tune the hyperparameters of the model learning to achieve low squared-Bellman error on a set of validation transitions. We report mean absolute error (MAE) of state-value estimates, with ground truth values computed via extensive on-policy rollouts. A detailed list of hyperparameters is provided in Appendix~\ref{sec:appendix_hypers}.

To isolate the effect of evaluation-aware learning without confounding it with dynamics mismatch, we design the assessment environment to share the same state and action spaces, transition dynamics, and reward function as the deployment environment, without discounting ($\gamma_{\text{ae}} = 1$); we study the effect of dynamics differences between assessment and deployment separately in Section~\ref{sec:mismatch}. We select 5 assessment start states by sampling states visited during rollouts of a base A2C policy trained using standard policy gradient. This ensures the assessment states are reachable under reasonable behavior and from the region of the state space relevant to learning, while keeping the per-update assessment cost low. We fix the assessment horizon to 10 steps for discrete-action environments and 25 steps for continuous-action environments. Short assessment horizons further limit the cost of each assessment---a key practical consideration since assessments are performed at every policy update---while still allowing the policy to exhibit value-informative behavior within the assessment scenarios. Example assessment start-states for MinAtar environments are shown in Appendix~\ref{sec:appendix_assessment_env_states}.

\subsection{Performance--evaluability tradeoff}

We first establish that the performance--evaluability tradeoff predicted by Proposition~\ref{thm:monotonicity_of_pred_gap_returns} manifests empirically, and that co-learning the evaluator alongside the policy can mitigate it. We train a transformer-based evaluator on data collected from all policies encountered during a standard A2C run of 10M interactions, then freeze it. Starting from a randomly initialized policy, we perform evaluation-aware policy learning under this frozen evaluator with varying~$\beta$ and track the resulting returns and evaluation errors.

Figure~\ref{fig:evarl_with_pre_trained_predictor} shows the results across three discrete-action environments. With the frozen evaluator, increasing $\beta$ consistently decreases MAE but also reduces returns, confirming the tradeoff predicted by Proposition~\ref{thm:monotonicity_of_pred_gap_returns}. This is expected: as the evaluator is frozen, the only way the policy can reduce evaluation error is by moving toward regions where the evaluator happens to be accurate, potentially at the expense of higher-performing but harder-to-evaluate behaviors.

To test whether co-learning can mitigate this tradeoff, we repeat the experiment but allow the evaluator to update alongside the policy. As shown in Figure~\ref{fig:evarl_with_pre_trained_predictor}, co-learning achieves returns close to the standard RL baseline while maintaining low evaluation error. This demonstrates that the tradeoff observed with a frozen evaluator is not inherent to the EvA-RL framework---rather, it reflects the limitations of a fixed evaluator. By co-learning the evaluator alongside the policy, the burden of reducing evaluation error partially shifts onto the evaluator, allowing the policy to focus more on performance maximization.

\begin{figure}[h]
  \centering
  \begin{minipage}{0.55\linewidth}
    \centering
    \includegraphics[width=\textwidth]{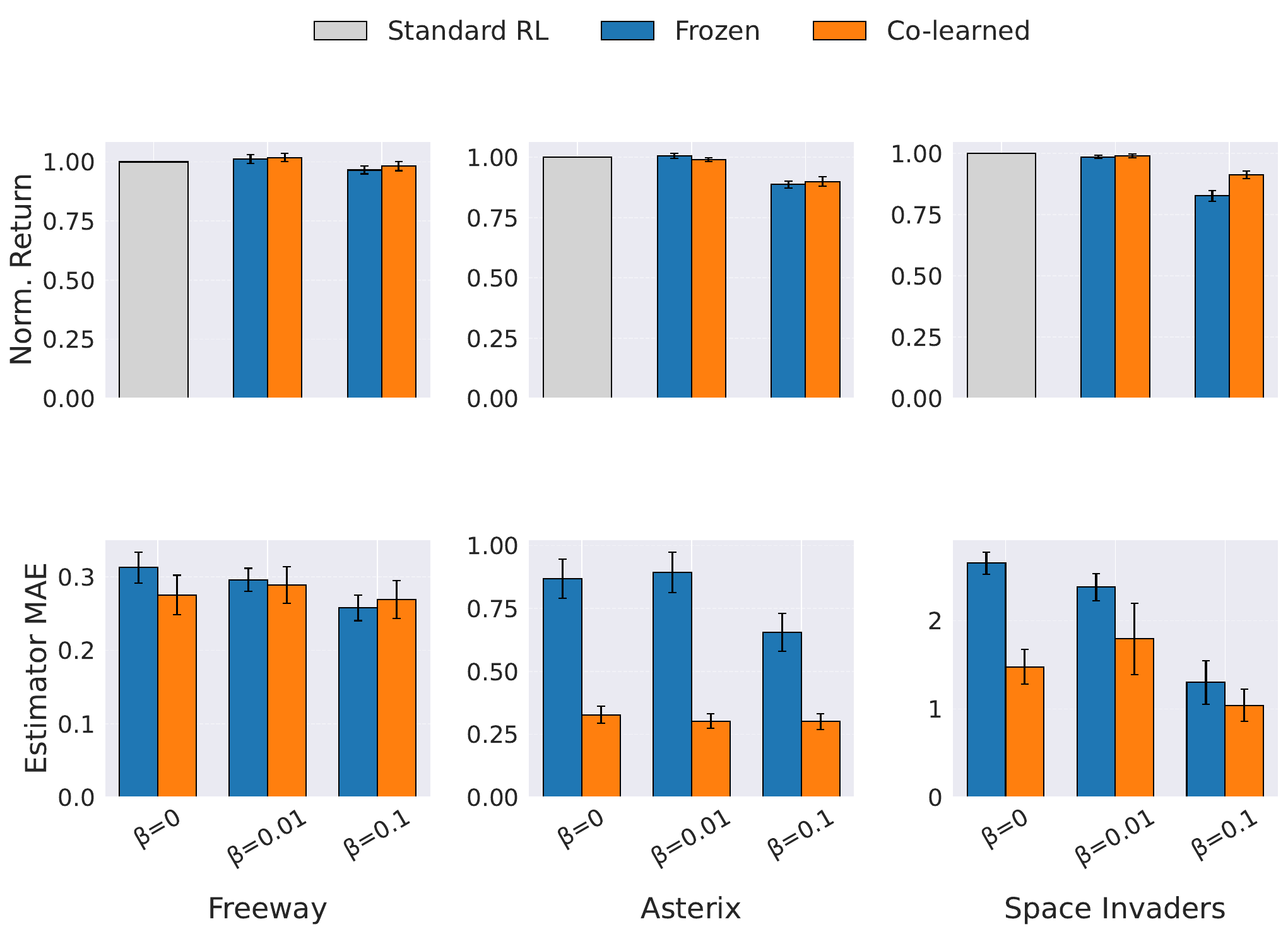}
  \end{minipage}%
  \hfill
  \begin{minipage}{0.4\linewidth}
    \captionof{figure}{\textit{Performance--evaluability tradeoff and the effect of co-learning.} Three conditions are shown per environment: \textit{baseline} (standard RL, $\beta = 0$), \textit{frozen} (EvA-RL with a pre-trained, frozen evaluator at varying $\beta$), and \textit{co-learned} (EvA-RL with the evaluator updated alongside the policy at varying $\beta$). We report mean and standard error across 20 random seeds. Top: normalized returns relative to the baseline. Bottom: value estimation MAE. With a frozen evaluator, increasing $\beta$ reduces MAE at the cost of lower returns, confirming Proposition~\ref{thm:monotonicity_of_pred_gap_returns}. Co-learning achieves returns close to the baseline while generally maintaining low evaluation error, demonstrating that the tradeoff is not inherent to EvA-RL but reflects the limitations of a fixed evaluator.}
    \label{fig:evarl_with_pre_trained_predictor}
  \end{minipage}
\end{figure}

\subsection{Co-learned evaluator vs.\ standard OPE methods}\label{sec:colearnd_evaluator_vs_ope_on_eva}

Having established that co-learning mitigates the performance--evaluability tradeoff, we now test whether the co-learned evaluator produces more accurate value estimates of the evaluation-aware policy than standard OPE methods given access to the same logged training data. We perform EvA-RL with co-learned evaluator for $\beta \in \{0.01, 0.1\}$. All OPE baselines (FQE, TIS, PDIS, DR) are provided with the deployment trajectories logged during EvA-RL training; assessment rollouts are used only by the co-learned evaluator (see Appendix~\ref{sec:appendix_ope_data} for details).

Figure~\ref{fig:evarl_ope_discrete} shows the comparison of the quality of value estimates of the final EvA-RL policy. The co-learned evaluator consistently achieves lower MAE than all OPE baselines across all three discrete-action environments. This advantage stems from the co-learning dynamic as opposed to OPE using a fixed strategy.

\begin{figure}[h]
  \centering
  \includegraphics[width=.65\linewidth]{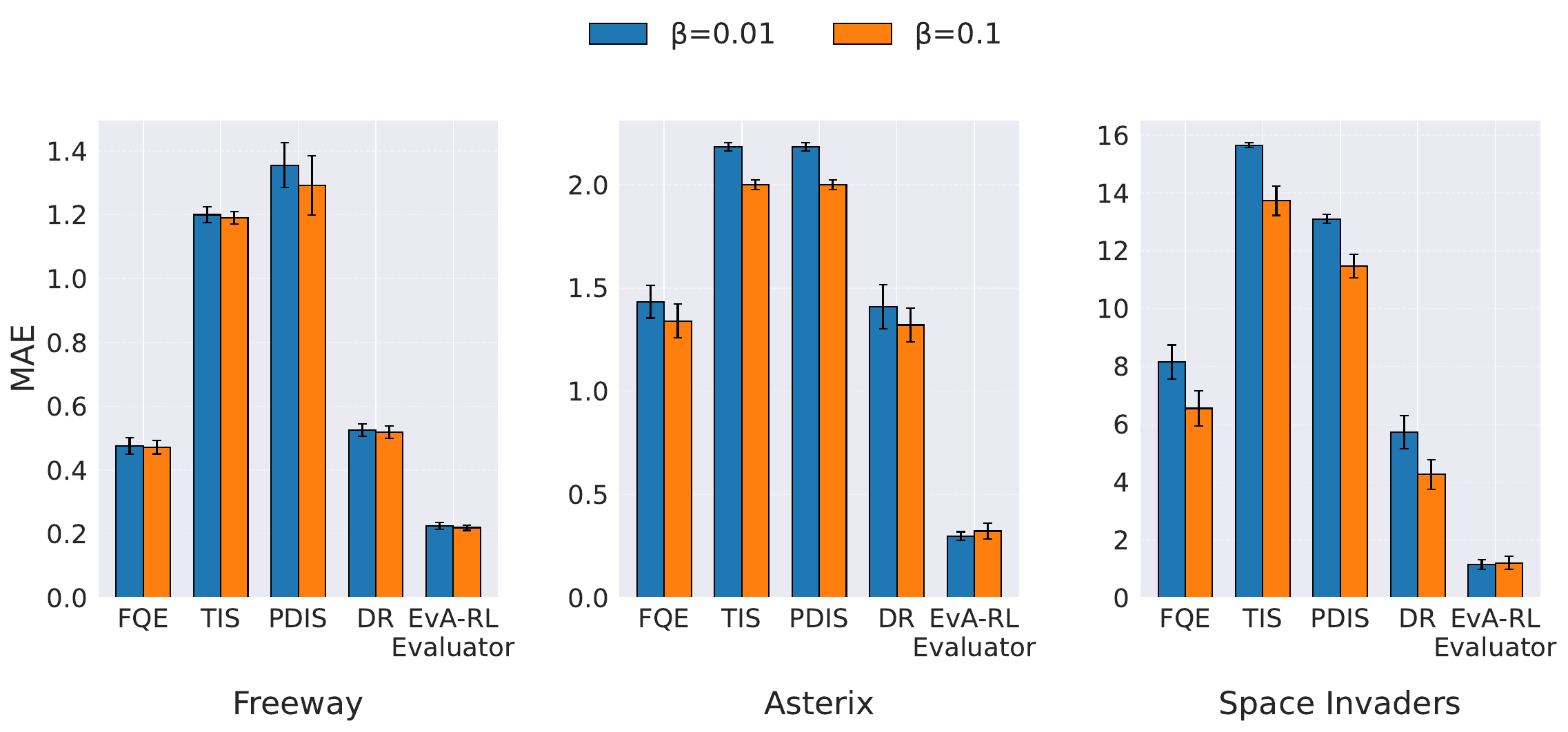}
  \caption{\textit{Co-learned evaluator vs.\ OPE methods on EvA-RL policies.} MAE of state-value estimates across three discrete-action environments for $\beta \in \{0.01, 0.1\}$. Both the co-learned evaluator and all OPE baselines (FQE, TIS, PDIS, DR) have access to the same deployment trajectories logged during EvA-RL training. We report mean and standard error across 20 seeds. The co-learned evaluator consistently achieves lower evaluation error on EvA-RL policies than all OPE baselines.}
  \label{fig:evarl_ope_discrete}
\end{figure}

\subsection{End-to-end pipeline comparison}

\begin{figure}[h]
    \centering
    \includegraphics[width=.65\linewidth]{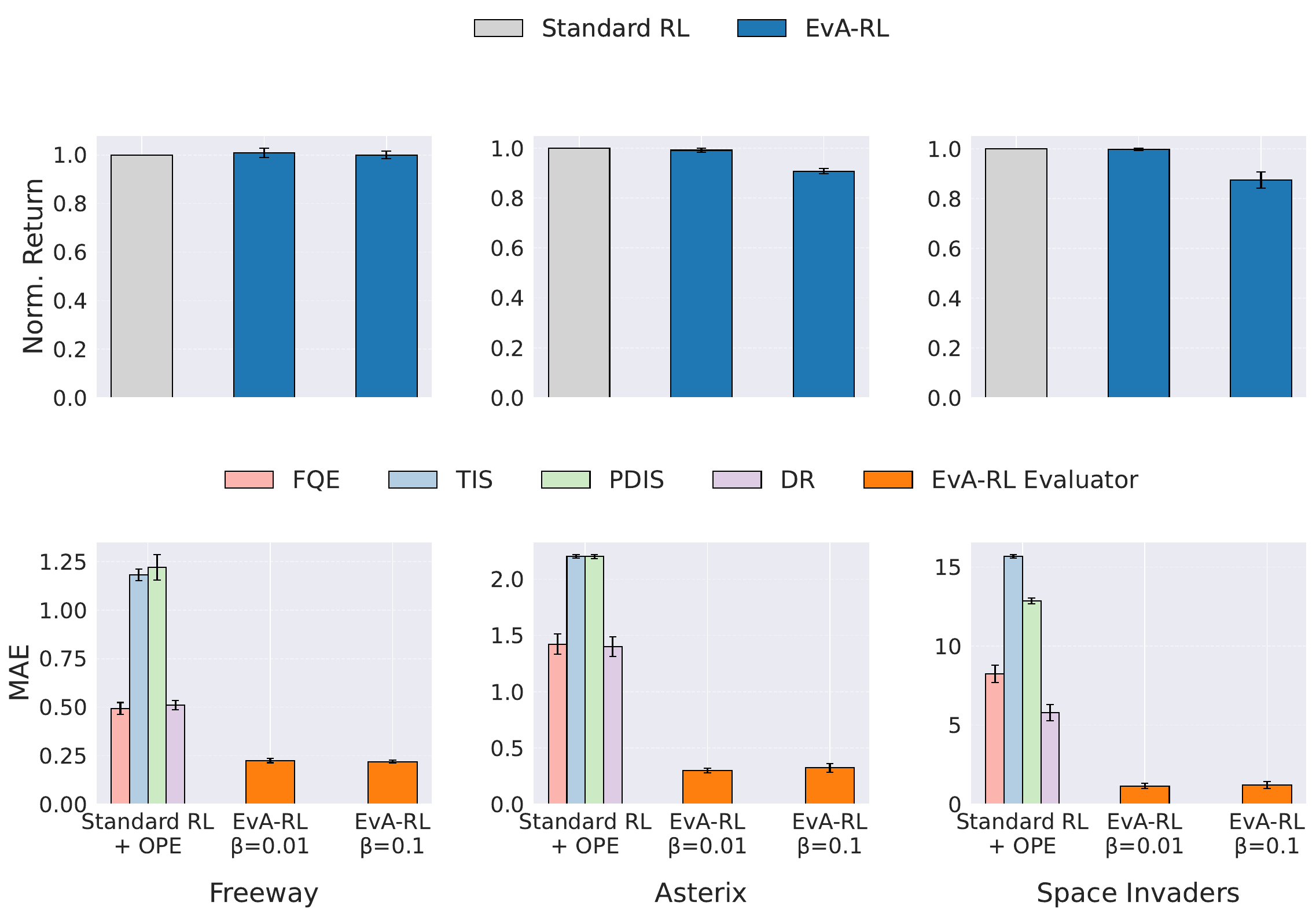}
    \captionof{figure}{\textit{Standard RL + OPE vs.\ EvA-RL in discrete-action environments.} A policy trained with standard A2C is evaluated using four OPE baselines (FQE, TIS, PDIS, DR); a policy trained with EvA-RL is evaluated using its co-learned evaluator. Each pipeline uses the data produced by its own training procedure. Top: normalized discounted returns (relative to standard RL; higher is better). Bottom: MAE of state-value estimates (lower is better). Mean and standard error across 20 seeds. EvA-RL with $\beta=0.01$ nearly matches standard RL returns while reducing evaluation error.}
    \label{fig:evarl_discrete}
\end{figure}

The previous experiment tested whether the co-learned evaluator is better adapted to its policy than OPE methods applied to the same EvA-RL policy. We now add evaluation-aware policy learning and compare full pipelines end-to-end: a policy trained with standard A2C and evaluated using OPE methods, versus a policy trained with EvA-RL and evaluated using its co-learned evaluator. Each pipeline uses the data produced by its own training procedure: OPE baselines use all deployment trajectories logged during standard RL training, while the EvA-RL evaluator uses the same number of deployment trajectories from EvA-RL training, along with the short assessment rollouts (10--25 steps each) that are already collected as part of the EvA-RL training loop. The additional assessment rollouts amount to less than 1\% of the per-update interaction budget: 50 assessment transitions versus 6{,}400 deployment transitions per update in discrete environments, and 125 versus 20{,}480 in continuous environments.

Figure~\ref{fig:evarl_discrete} shows the results on discrete-action environments. EvA-RL with $\beta=0.01$ nearly matches the returns of standard RL while reducing evaluation error relative to all OPE baselines applied to the standard RL policy. This confirms that the evaluation gains observed in Section~\ref{sec:colearnd_evaluator_vs_ope_on_eva} carry over to the full pipeline setting, where both the policy and the evaluator differ.

\begin{table}[H]
    \centering
    \small
    \caption{\textit{EvA-RL vs.\ standard RL + OPE in continuous-action environments.} A policy trained with standard A2C is evaluated using three OPE baselines (FQE, PDIS, DR). A policy trained with EvA-RL is evaluated using its co-learned evaluator. Each pipeline uses the data produced by its own training procedure. Top: MAE of state-value estimates (lower is better). Bottom: normalized returns of EvA-RL relative to standard A2C. Mean $\pm$ standard error across 20 seeds.}
    \label{tab:continuous}
    \begin{tabular}{lccc}
        \toprule
        & \textbf{HalfCheetah} & \textbf{Reacher} & \textbf{Ant} \\
        \midrule
        \multicolumn{4}{l}{\textit{Value estimation MAE ($\downarrow$)}} \\[2pt]
        Standard RL w/ FQE  & $29.19 \pm 4.76$  & $62.22 \pm 3.49$  & $46.58 \pm 5.94$ \\
        Standard RL w/ PDIS & $5.28 \pm 0.66$   & $37.33 \pm 8.12$  & $28.38 \pm 5.48$ \\
        Standard RL w/ DR   & $7.29 \pm 1.94$   & $35.30 \pm 2.87$  & $\mathbf{15.80 \pm 5.74}$ \\
        EvA-RL w/ Evaluator & $\mathbf{4.18 \pm 0.49}$   & $\mathbf{13.90 \pm 0.26}$  & $16.47 \pm 6.61$ \\
        \midrule
        \multicolumn{4}{l}{\textit{Normalized return of EvA-RL ($\uparrow$)}} \\[2pt]
        EvA-RL / Standard A2C & $1.052\pm 0.005$ & $1.000 \pm 0.001$ & $0.999 \pm 0.004$ \\
        \bottomrule
    \end{tabular}
\end{table}

We further test whether these benefits extend to continuous control. Table~\ref{tab:continuous} reports results on HalfCheetah, Reacher, and Ant with $\beta = 5 \times 10^{-4}$. TIS is excluded due to exponential variance at the task horizon of $T = 1000$. EvA-RL achieves normalized returns comparable to standard A2C while the co-learned evaluator produces more accurate value estimates than FQE, PDIS, and DR on HalfCheetah and Reacher, and outperforms FQE and PDIS on Ant, where DR achieves a slightly lower MAE. These results demonstrate that EvA-RL provides competitive performance with more reliable evaluation across both discrete and continuous domains. Learning curves for all Gymnax and Brax environments are provided in Appendix~\ref{sec:appendix_addl_experimental_results}.

\begin{figure}[h]
\centering
\begin{minipage}[t]{0.6\linewidth}
\vspace{0pt}
\centering
\includegraphics[width=\linewidth]{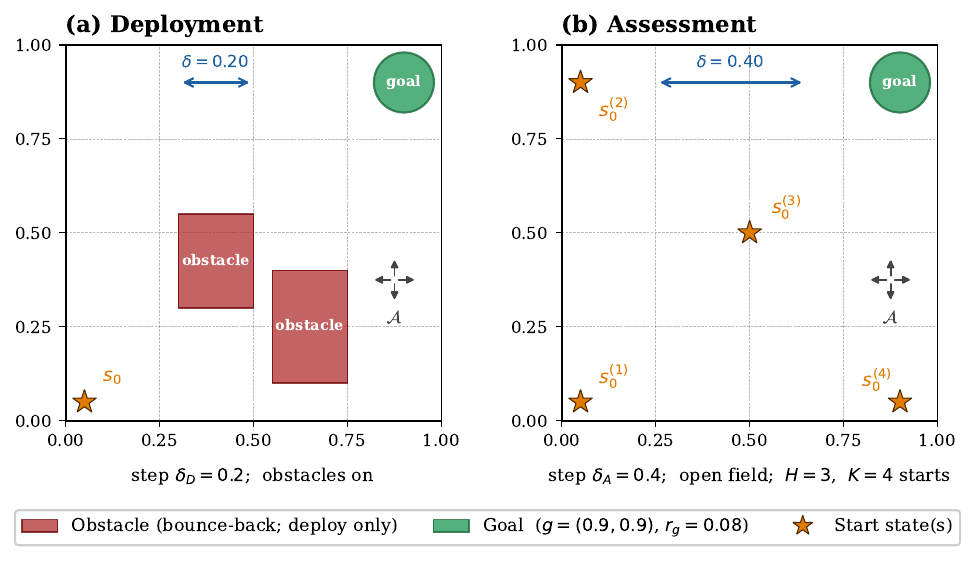}
\caption{\textit{Deployment and assessment environments for the dynamics mismatch experiment.} \textbf{(a)} The deployment environment is a unit-square grid with two rectangular obstacles and step size $\delta_D = 0.2$. The agent starts at $s_0$ (bottom left) and must reach the goal (top right); obstacle collisions result in bounce-back. \textbf{(b)} The assessment environment is an obstacle-free open field with step size $\delta_A = 0.4$. The policy is rolled out for $H = 3$ steps from each of $K = 4$ fixed start states spread across the unit square.}
\label{fig:env_diagram}
\end{minipage}
\hfill
\begin{minipage}[t]{0.38\linewidth}
\vspace{0pt}
\captionof{table}{\textit{EvA-RL vs.\ standard RL under dynamics mismatch.} Deployment: unit-square grid with obstacles, step size $0.2$, dense distance-to-goal reward. Assessment: obstacle-free open field, step size $0.4$, horizon 3 steps per start state. Top: MAE of state-value estimates (lower is better, best in \textbf{bold}). Bottom: mean episodic return. Mean $\pm$ standard error across 30 seeds. Despite the dynamics gap, EvA-RL reduces evaluation error while maintaining competitive returns.}
\label{tab:mismatch}
\centering
\footnotesize
\begin{tabular}{lc}
\toprule
 & \textbf{2D Navigation} \\
\midrule
\multicolumn{2}{l}{\textit{Value estimation MAE ($\downarrow$)}} \\
Standard RL w/ FQE  & $4.03 \pm 0.29$ \\
Standard RL w/ PDIS & $2.11 \pm 0.06$ \\
Standard RL w/ DR   & $2.31 \pm 0.02$ \\
EvA-RL w/ FQE       & $4.13 \pm 0.27$ \\
EvA-RL w/ PDIS      & $1.97 \pm 0.07$ \\
EvA-RL w/ DR        & $2.21 \pm 0.02$ \\
EvA-RL w/ Evaluator & $\mathbf{1.65 \pm 0.18}$ \\
\midrule
\multicolumn{2}{l}{\textit{Episodic return ($\uparrow$)}} \\
Standard RL             & $-4.61 \pm 0.21$ \\
EvA-RL ($\beta = 0.05$) & $-4.78 \pm 0.33$ \\
\bottomrule
\end{tabular}
\end{minipage}
\end{figure}
\subsection{Dynamics mismatch between assessment and deployment}
\label{sec:mismatch}

The experiments above used an assessment environment with shared dynamics. In practice, assessment environments are often simplified approximations of deployment --- designed to elicit value-informative behavior cheaply and safely. We now ask whether EvA-RL remains effective under such differences. We construct a 2D navigation task (continuous state space, four discrete actions) in which the deployment environment is a unit-square grid with rectangular obstacles and step size $0.2$, while the assessment environment is an obstacle-free open field with step size $0.4$ (Figure~\ref{fig:env_diagram}). The agent receives a dense reward of $-\|s' - s_{\text{goal}}\|$ at each step, plus a bonus of $+1$ on reaching the goal. With 4 start states and a horizon of 3 steps each, assessment rollouts amount to 12 transitions per update --- less than 1\% of the 1{,}600 deployment transitions collected per update. We train for 1000 policy update iterations (1.6M environment steps total) with a warm-up phase of 100 updates and $\beta = 0.05$, averaging results over 30 seeds. Additional hyperparameters are provided in Appendix~\ref{sec:appendix_dynamics_mismatch}.

Table~\ref{tab:mismatch} reports value estimation MAE and returns with mean and standard error. Despite the dynamics gap, the co-learned EvA-RL evaluator achieves a lower MAE ($1.65 \pm 0.18$) than all OPE baselines applied to either the standard RL or the EvA-RL policy, while the EvA-RL policy maintains returns comparable to standard RL ($-4.78 \pm 0.33$ vs.\ $-4.61 \pm 0.21$). This demonstrates that co-learning can compensate for a moderate dynamics mismatch: the evaluator learns to map open-field behavior to deployment values even when the two environments differ in step size and obstacle structure.
\section{Conclusion}

We introduced EvA-RL, a framework that incorporates evaluation accuracy into the policy learning objective, as opposed to the standard approach of evaluating a policy post-hoc after training. The key idea is that a policy can be shaped during learning to be not only performant but also amenable to accurate evaluation. This requires an evaluation scheme that can be queried efficiently during training, which we instantiate through assessment environments and a co-learned value evaluator. Our experiments demonstrate that, at the cost of short additional assessment rollouts during training, EvA-RL achieves competitive returns while generally producing lower evaluation error than standard off-policy evaluation methods and that these benefits persist even when the assessment environment differs from deployment in dynamics and obstacle structure.

\paragraph{Limitations.} Assessment start states were sampled heuristically and the warm-up phase length was selected manually; both could be replaced by systematic procedures. The value evaluator conditioned only on assessment start states and returns; richer inputs such as full trajectories could improve accuracy in more complex settings.

\paragraph{Future work.} This work opens several directions, including principled design of assessment environments for low-cost on-policy evaluation, reliable policy performance estimation in safety-critical domains where standard OPE methods are prone to misestimation, and assessment-based deployment certification.

\bibliography{mybib.bib}
\bibliographystyle{iclr2026_conference}

\appendix
\section{Compatibility conditions for a policy to operate in both deployment and assessment MDPs}\label{sec:compatibility_conditions}
For a policy $\pi$ to operate in both the deployment environment $\mathcal{M}_D$ and the assessment environment $\mathcal{M}_A$, the following conditions must be satisfied: (i)~the action space of the assessment environment is a subset of that of the deployment environment, i.e., $\mathcal{A}_A \subseteq \mathcal{A}_D$, so that any action available in $\mathcal{M}_A$ is also meaningful in $\mathcal{M}_D$; and (ii)~states in both environments admit a common representation, i.e., there exists a shared feature space $\mathcal{X}$ and mappings $\phi_D: \mathcal{S}_D \to \mathcal{X}$, $\phi_A: \mathcal{S}_A \to \mathcal{X}$ such that $\pi$ is defined over $\mathcal{X}$. A simple sufficient condition for both is that $\mathcal{S}_A \subseteq \mathcal{S}_D$ and $\mathcal{A}_A \subseteq \mathcal{A}_D$, which holds when the assessment environment is a restriction of the deployment environment to a subset of states or scenarios.

\section{Theorems and proofs}




\subsection{Performance--evaluability tradeoff with a fixed evaluator}
\label{sec:monotonicity_pred_gap_returns}

We first establish two general lemmas about optimization problems of the form $\max_x \{f(x) - \beta\, g(x)\}$, then apply them to the EvA-RL objective.

\begin{lemma}\label{thm:appendix_g_monotonicity}
    Let $x^*(\beta) \in \arg\max_x \{f(x) - \beta\, g(x)\}$ with $\beta \geq 0$. Then $g(x^*(\beta))$ is non-increasing in $\beta$: for $\beta_2 > \beta_1 \geq 0$, $g(x^*(\beta_2)) \leq g(x^*(\beta_1))$.
\end{lemma}
\begin{proof}
    Let $x_1 = x^*(\beta_1)$ and $x_2 = x^*(\beta_2)$ with $\beta_2 > \beta_1$. Optimality gives:
    \[
        f(x_1) - \beta_1 g(x_1) \geq f(x_2) - \beta_1 g(x_2), \qquad f(x_2) - \beta_2 g(x_2) \geq f(x_1) - \beta_2 g(x_1).
    \]
    Suppose for contradiction that $g(x_2) > g(x_1)$. Adding the two inequalities yields:
    \[
        0 \geq (\beta_2 - \beta_1)(g(x_2) - g(x_1)) > 0,
    \]
    a contradiction. Hence $g(x_2) \leq g(x_1)$.
\end{proof}

\begin{lemma}\label{thm:appendix_f_monotonicity}
    Under the same setup, $f(x^*(\beta))$ is non-increasing in $\beta$: for $\beta_2 > \beta_1 \geq 0$, $f(x^*(\beta_2)) \leq f(x^*(\beta_1))$.
\end{lemma}
\begin{proof}
    With $x_1, x_2$ as above, Lemma~\ref{thm:appendix_g_monotonicity} gives $g(x_1) \geq g(x_2)$. Optimality at $\beta_1$ implies:
    \[
        f(x_1) - \beta_1 g(x_1) \geq f(x_2) - \beta_1 g(x_2) \implies f(x_1) - f(x_2) \geq \beta_1(g(x_1) - g(x_2)) \geq 0.
    \]
    Hence $f(x_2) \leq f(x_1)$.
\end{proof}

\begin{proposition}\label{thm:appendix_monotonicity_of_pred_gap_returns}
    Increasing $\beta$ in the EvA-RL objective (Eq.~\ref{eq:value_predictability_objective}) with a fixed evaluator leads to non-increasing evaluation error $\mathcal{L}_\mathrm{MSE}$ and non-increasing expected return $J_\pi$ of the resulting policy.
\end{proposition}
\begin{proof}
    The EvA-RL objective has the form $\max_\pi \{f(\pi) - \beta\, g(\pi)\}$ with $f(\pi) = J_\pi$ and $g(\pi) = \mathcal{L}_\mathrm{MSE}(\pi, \hat{V})$. Applying Lemmas~\ref{thm:appendix_g_monotonicity} and~\ref{thm:appendix_f_monotonicity} directly yields the result. Note that the lemmas hold for any feasible set over which the maximization is performed, including one defined by Bellman consistency constraints. Therefore the result applies both to the relaxed setting where state-values are free in $\mathbb{R}^n$ and to the practical setting where they must correspond to valid policies.
\end{proof}

\subsection{Non-uniqueness of optimal value functions in EvA-RL}
\label{sec:appendix_non_uniqueness}

As an additional structural observation, we note that the EvA-RL objective can admit multiple optimal value functions under a fixed estimator — a departure from standard RL, where all optimal policies share the same optimal value function~\citep{SuttonBarto1998}. We establish this formally below under specific assumptions: a finite state set, shared dynamics between deployment and assessment environments, and a linear transformer-based estimator~\citep{Katharopoulos2020TransformersRNNs}.

\subsubsection{Setup and assumptions}
\label{sec:appendix_multiplicity_setup}

For this analysis, we consider the setting where the assessment environment is a restriction of the deployment environment: both share the same state set $\mathcal{S}$, action set $\mathcal{A}$, transition dynamics $P$, and reward function $R$, with $|\mathcal{S}| = n$ finite. Without loss of generality, the first $k$ states are the assessment start states. Since the dynamics are shared, $V_A^\pi(s^i) = V_D^\pi(s^i)$ for all assessment start states $s^i$.

We consider a sub-class of evaluators that produce a state-value estimate as a weighted average of assessment state-values, with weights determined by the similarity between the evaluation state and the assessment start states. Concretely, we use a linear transformer-based~\citep{Katharopoulos2020TransformersRNNs} function approximator:
\begin{equation}
    \label{eq:similarity_based_predictability_value_estimator}
    \hat{V}_D^{\pi, \phi}(s \mid \Xi_A) = \frac{\sum_{i=1}^{k} \phi(s)^T \phi(s^i) V_D^\pi(s^i)}{\sum_{i=1}^{k} \phi(s)^T \phi(s^i)},
\end{equation}
where $\phi: \mathcal{S} \to \mathbb{R}^{d'}$ is a state-embedding function with positive inner products. We write $\hat{V}_D^{\pi, \phi}$ to make the dependence on $\phi$ explicit and assume that $\phi$ is fixed throughout the optimization. The mean squared error of a policy $\pi$ with respect to this evaluator is $\mathcal{L}_\mathrm{MSE}(\pi, \phi) = \sum_{s \in \mathcal{S}} \mu_D(s) \big(V_D^\pi(s) - \hat{V}_D^{\pi, \phi}(s \mid \Xi_A)\big)^2$.

\subsubsection{Vectorized formulation}
\label{sec:appendix_vectorized}

To leverage linear algebraic tools, we recast the EvA-RL objective in vectorized form. Let $\mu_D = [\mu_D(s_1), \ldots, \mu_D(s_n)]^T$ and $V_D^{\pi} = [V_D^{\pi}(s_1), \ldots, V_D^{\pi}(s_n)]^T$. We define a similarity function $f: \mathcal{S} \times \mathcal{S} \to \mathbb{R}^+$ as $f(s, s') = \phi(s)^T \phi(s')$, so that the evaluator becomes:
\begin{equation}
    \hat{V}_D^{\pi, \phi}(s \mid \Xi_A) = \frac{\sum_{s^i \in \mathcal{S}_A} f(s, s^i) \, V_D^\pi(s^i)}{\sum_{s^i \in \mathcal{S}_A} f(s, s^i)}.
\end{equation}

We construct an $n \times n$ matrix $F$ with entries $F_{ij} = f(s_i, s^j)$ for $j \leq k$ and $F_{ij} = 0$ for $j > k$:
\begin{equation}
    F = \begin{bmatrix}
        f(s^1, s^1) & \cdots & f(s^1, s^k) & 0 & \cdots & 0 \\
        \vdots & \ddots & \vdots & \vdots & \ddots & \vdots \\
        f(s^k, s^1) & \cdots & f(s^k, s^k) & 0 & \cdots & 0 \\
        f(s_{k+1}, s^1) & \cdots & f(s_{k+1}, s^k) & 0 & \cdots & 0 \\
        \vdots & \ddots & \vdots & \vdots & \ddots & \vdots \\
        f(s_n, s^1) & \cdots & f(s_n, s^k) & 0 & \cdots & 0
    \end{bmatrix}_{n \times n}
\end{equation}
where we use superscripted $s$ for assessment start states and subscripted $s$ for the remaining states. The vector of value estimates is then:
\begin{equation}
    \hat{V}_D^{\pi, \phi} = \mathrm{diag}(F \mathbf{1}_n)^{-1} F V_D^{\pi},
\end{equation}
where $\mathbf{1}_n$ is a vector of ones of size $n$. Defining $Q = (I - \mathrm{diag}(F \mathbf{1}_n)^{-1} F)^T \mathrm{diag}(\mu_D)(I - \mathrm{diag}(F \mathbf{1}_n)^{-1} F)$, the MSE can be written compactly as:
\begin{equation}
    \label{eq:mse_in_vector_form}
    \mathcal{L}_\mathrm{MSE}(\pi, \phi) = V_D^{\pi T} Q V_D^{\pi}.
\end{equation}

The soft and hard EvA-RL objectives become, respectively:
\begin{align}
    \label{eq:vectorized_soft}
    &\max_\pi \quad \mu_D^T V_D^\pi - \beta \, V_D^{\pi T} Q V_D^\pi,\; \beta > 0 \\
    \label{eq:vectorized_hard}
    &\max_\pi \quad \mu_D^T V_D^\pi \quad \text{s.t.} \quad V_D^{\pi T} Q V_D^\pi \leq \epsilon^2,\; \epsilon > 0
\end{align}

\subsubsection{Convexity of the hard-constrained problem}
\label{sec:appendix_convexity}

\begin{lemma}\label{thm:convexity_of_optimization_problem}
    The hard-constrained EvA-RL problem, with state-values treated as free variables in $\mathbb{R}^n$ (i.e., without enforcing Bellman consistency):
    \begin{equation}
        \label{eq:hard_constrained_real_valued_values_pred_objective}
        \max_{V_D^\pi \in \mathbb{R}^n} \; \mu_D^T V_D^{\pi} \quad \text{s.t.} \quad V_D^{\pi T} Q V_D^{\pi} \leq \epsilon^2
    \end{equation}
    is a convex optimization problem. Specifically, it is a quadratically constrained linear program (QCLP).
\end{lemma}
\begin{proof}
    The objective $\mu_D^T V_D^{\pi}$ is affine in $V_D^\pi$. For the constraint, observe that for all $V_D^{\pi} \in \mathbb{R}^n$:
    \begin{equation}
        V_D^{\pi T} Q V_D^{\pi} = V_D^{\pi T} (I - \mathrm{diag}(F \mathbf{1}_n)^{-1} F)^T \mathrm{diag}(\mu_D) (I - \mathrm{diag}(F \mathbf{1}_n)^{-1} F) V_D^{\pi} \geq 0,
    \end{equation}
    so $Q$ is positive semi-definite and the constraint set is convex. The problem is therefore a QCLP~\citep{BoydVandenberghe2004Convex}.
\end{proof}

\subsubsection{Upper bound on expected return}
\label{sec:appendix_upper_bound_relaxed_evarl}

The optimal value of a problem of the form $\max_{x} \; a^T x \;\text{s.t.}\; x^T Q x \leq \epsilon^2$, where $Q$ is positive semi-definite, is $a^T x^* = \epsilon \sqrt{a^T Q^{\dagger} a}$, where $Q^{\dagger}$ denotes the pseudo-inverse of $Q$. Applying this to the hard-constrained problem~\eqref{eq:hard_constrained_real_valued_values_pred_objective}, the maximum expected return without Bellman consistency is:
\begin{equation}
    \mu_D^T V_D^{\pi^*} = \epsilon \sqrt{\mu_D^T Q^{\dagger} \mu_D}.
\end{equation}
The expected return under Bellman consistency is strictly upper bounded by this quantity, since Bellman constraints further restrict the feasible set.

\subsubsection{Equivalence of hard and soft formulations}
\label{sec:appendix_relation_soft_hard}

\begin{theorem}\label{thm:appendix_beta_epsilon_relation}
    Any solution $V_D^{\pi^*_\epsilon}$ to the hard-constrained problem~\eqref{eq:hard_constrained_real_valued_values_pred_objective} is a solution to the soft-constrained problem
    \begin{equation}
        \label{eq:real_valued_soft_pred_objective}
        \max_{V_D^{\pi} \in \mathbb{R}^n} \quad \mu_D^T V_D^\pi - \beta \, V_D^{\pi T} Q V_D^{\pi},
    \end{equation}
    for $\beta = \mu_D^T V_D^{\pi^*_{\epsilon}} / (2\epsilon^2)$. Conversely, any solution $V_D^{\pi^*_{\beta}}$ to the soft-constrained problem is a solution to the hard-constrained problem with $\epsilon^2 = V_D^{\pi^*_{\beta} T} Q V_D^{\pi^*_{\beta}}$.
\end{theorem}
\begin{proof}
    The soft-constrained objective in vectorized form is:
    \begin{equation}
        \max_{V_D^{\pi} \in \mathbb{R}^n} \quad \mu_D^T V_D^{\pi} - \beta \, V_D^{\pi T} Q V_D^{\pi}.
    \end{equation}
    The first-order optimality condition gives:
    \begin{equation}
        \mu_D - 2\beta Q V_D^{\pi} = 0 \implies Q V_D^{\pi} = \frac{1}{2\beta} \mu_D.
    \end{equation}
    Let $V_D^{\pi^*_{\beta}}$ be a solution. Set $\epsilon^2 = V_D^{\pi^*_{\beta} T} Q V_D^{\pi^*_{\beta}}$. Suppose a global optimum $V_D^{\pi^*_{\epsilon}}$ of the hard-constrained problem satisfies $V_D^{\pi^*_{\epsilon}} \neq V_D^{\pi^*_{\beta}}$. Then:
    \begin{equation}
        \mu_D^T V_D^{\pi^*_{\epsilon}} - \beta \, V_D^{\pi^*_{\epsilon} T} Q V_D^{\pi^*_{\epsilon}} > \mu_D^T V_D^{\pi^*_{\beta}} - \beta \, V_D^{\pi^*_{\beta} T} Q V_D^{\pi^*_{\beta}},
    \end{equation}
    contradicting optimality of $V_D^{\pi^*_{\beta}}$ for the soft problem. Hence $V_D^{\pi^*_{\beta}}$ is also optimal for the hard-constrained problem.

    For the converse, the Lagrangian of the hard-constrained problem is:
    \begin{equation}
        \mathcal{L}(V_D^{\pi}, \lambda) = -\mu_D^T V_D^{\pi} + \lambda \big(V_D^{\pi T} Q V_D^{\pi} - \epsilon^2\big).
    \end{equation}
    At the optimum $V_D^{\pi^*_{\epsilon}}$, stationarity gives $Q V_D^{\pi^*_{\epsilon}} = \frac{1}{2\lambda} \mu_D$. Comparing with the soft optimality condition yields $\lambda = \beta$, and substituting back gives $\beta = \mu_D^T V_D^{\pi^*_{\epsilon}} / (2\epsilon^2)$.
\end{proof}

\begin{corollary}\label{thm:appendix_bellman_relaxed_soft_evarl}
    The soft-constrained EvA-RL objective (without Bellman consistency) admits multiple optimal value functions.
\end{corollary}
\begin{proof}
    From the first-order condition, any $V_D^{\pi} \in \mathbb{R}^n$ satisfying $Q V_D^{\pi} = \frac{1}{2\beta} \mu_D$ is optimal. Since $Q$ is positive semi-definite (but not positive definite in general), this system admits infinitely many solutions differing by vectors in the null space of $Q$.
\end{proof}

\begin{figure}[H]
  \centering
    \begin{subfigure}[t]{.28\textwidth}
        \centering
        \includegraphics[width=\linewidth]{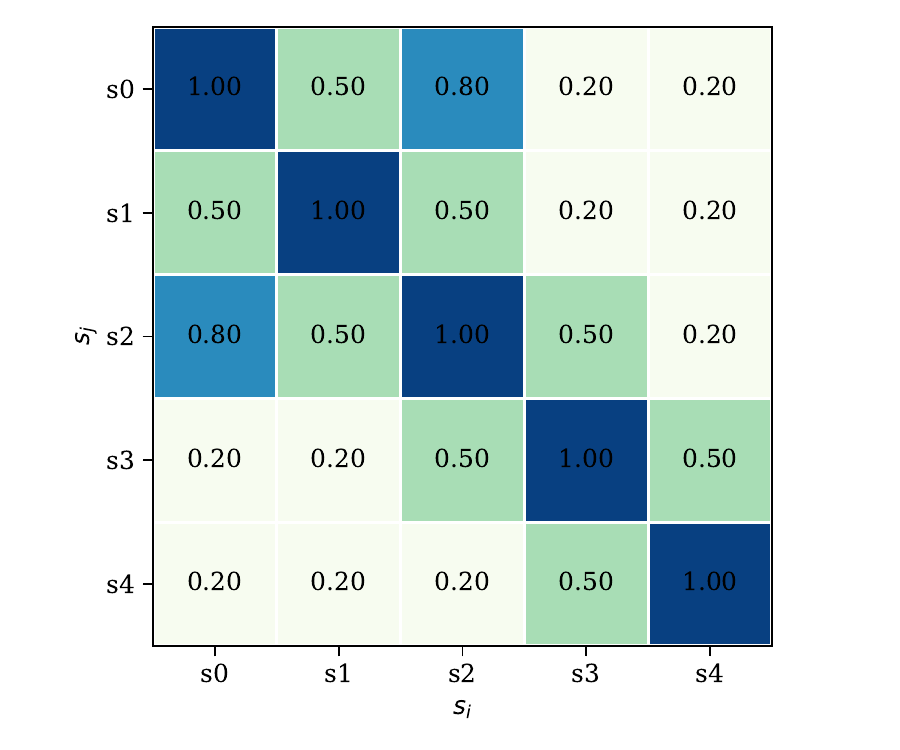}
        \caption{$\phi(s_i)^T \phi(s_j)$}
        \label{fig:state_sim}
    \end{subfigure}\hfill
  \begin{subfigure}[t]{0.35\textwidth}
    \centering
    \includegraphics[height=3cm]{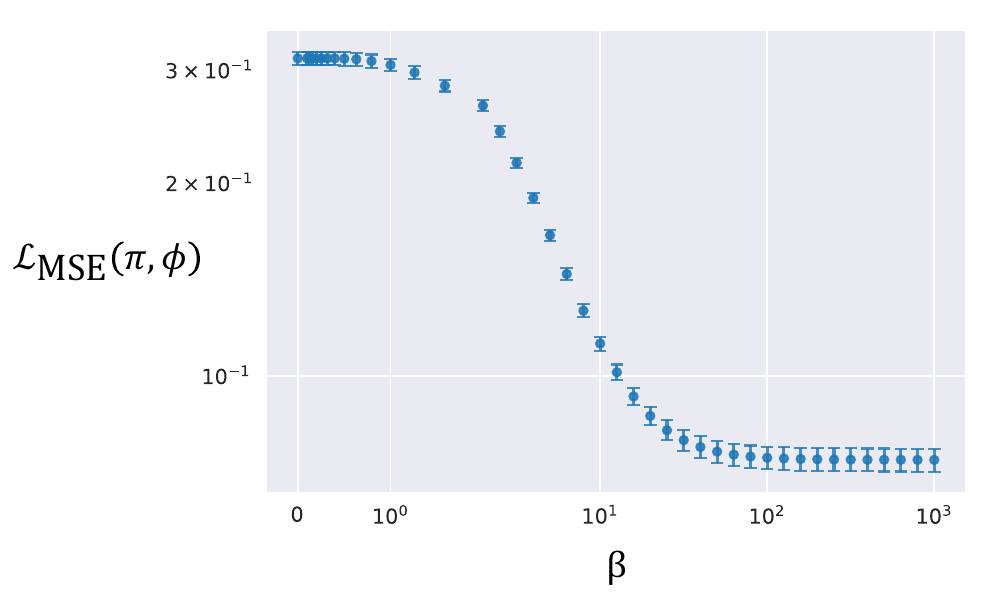}
    \caption{$\beta$ vs. $\mathcal{L}_\mathrm{MSE}(\pi, \phi)$}
    \label{fig:beta_vs_g}
  \end{subfigure}\hfill
  \begin{subfigure}[t]{0.35\textwidth}
    \centering
    \includegraphics[height=3cm]{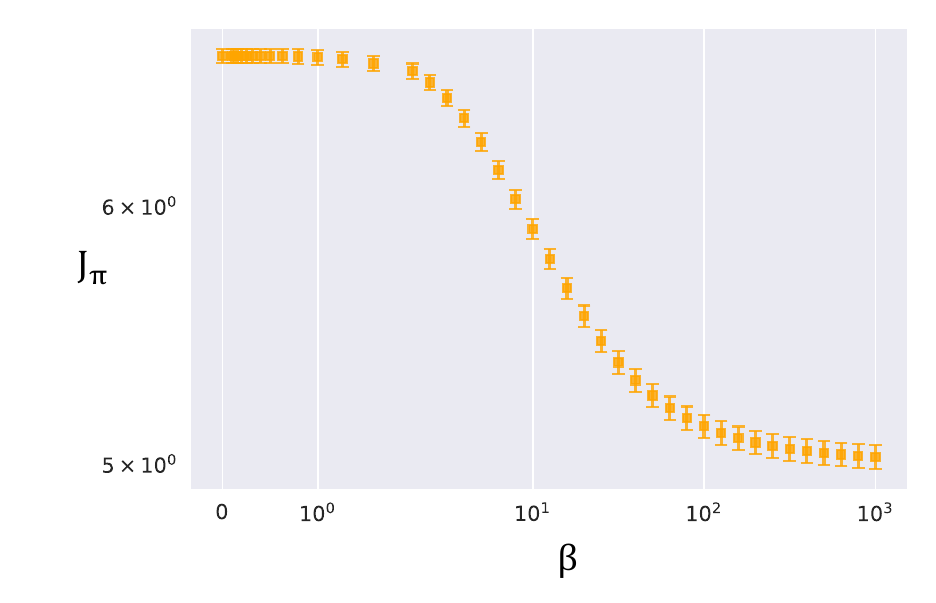}
    \caption{$\beta$ vs. $J_\pi$.}
    \label{fig:beta_vs_j}
  \end{subfigure}
  \caption{\textbf{Empirical verification of the performance--evaluability tradeoff (Proposition~\ref{thm:monotonicity_of_pred_gap_returns}).} \textit{Setup:} We construct a tabular MDP with 5 states and 2 actions per state. States $s_0$ and $s_2$ form the assessment environment, and the remaining states are non-assessment states. The state similarity function $f(s_i, s_j)$, used by the evaluator (Eq.~\ref{eq:similarity_based_predictability_value_estimator}), is fixed as shown in subfigure~\ref{fig:state_sim}. The start-state distribution $\mu_D$ is uniform over all states and the discount factor is $\gamma = 0.9$. \textit{Procedure:} For each trial, we randomly sample the transition and reward dynamics of the MDP. For each value of $\beta$, we train a softmax policy by optimizing the soft EvA-RL objective (Eq.~\ref{eq:value_predictability_objective}) and record the resulting evaluation error $\mathcal{L}_\mathrm{MSE}(\pi, \phi)$ and expected return $J_\pi$. Each data point is averaged over 1000 random trials; error bars denote standard error. \textit{Results:} Subfigures~\ref{fig:beta_vs_g} and~\ref{fig:beta_vs_j} confirm that both the evaluation error and the expected return are non-increasing in $\beta$, consistent with Proposition~\ref{thm:monotonicity_of_pred_gap_returns}.}
  \label{fig:beta_effect}
\end{figure}

\clearpage

\section{JAX implementation of Value Estimation Transformer}
\label{sec:appendix_pred_transformer}

\begin{lstlisting}[language=Python, basicstyle=\scriptsize\ttfamily, commentstyle=\color{gray}, keywordstyle=\color{blue}, stringstyle=\color{red}, numberstyle=\scriptsize, numbersep=5pt, frame=single, breaklines=true, showstringspaces=false]
import jax.numpy as jnp
import flax.linen as nn

class ValueEvaluationHead(nn.Module):

    num_heads:  int
    hidden_dim: int
    num_layers: int

    @nn.compact
    def __call__(
        self,
        assess_env_start_states:  jnp.ndarray,   # [B,k,obs]
        assess_env_returns: jnp.ndarray,   # [B,k]
        query_state:    jnp.ndarray    # [B,obs]
    ) -> jnp.ndarray:                  # [B]

        batch_size, k, obs_dim = assess_env_start_states.shape
        h = self.hidden_dim
        
        # --------------------------------------------------------
        # Token embeddings
        # --------------------------------------------------------
        assess_state_tokens  = nn.Dense(h)(assess_env_start_states)  # [B,k,h]
        assess_return_tokens = nn.Dense(h)(assess_env_returns[..., None])  #[B,k,h]
        query_state_token = nn.Dense(h)(query_state)[:, None, :]  # [B,1,h]

        # --------------------------------------------------------
        # Positional embeddings
        # --------------------------------------------------------
        pos_embeddings = nn.Embed(num_embeddings=k + 1, 
                        features=h)(jnp.arange(k + 1))  # [k + 1,h]
        
        assess_state_tokens = assess_state_tokens + pos_embeddings[:-1, :]  #[B,k,h]
        assess_return_tokens = assess_return_tokens + pos_embeddings[-1, :]  # [B,k,h]
        query_state_token = query_state_token + pos_embeddings[-1, None, :]  # [B,1,h]

        # --------------------------------------------------------
        # Full Sequence
        # --------------------------------------------------------
        full_sequence = jnp.concatenate([assess_state_tokens, 
        assess_return_tokens, query_state_token], axis=1)    # [B,2k+1,h]

        # --------------------------------------------------------
        # Transformer
        # --------------------------------------------------------
        x = full_sequence
        for _ in range(self.num_layers):
            # Self-attention block
            y = nn.LayerNorm()(x)
            y = nn.SelfAttention(num_heads=self.num_heads,
                                 qkv_features=h)(y)
            x = x + y
            # Feed-forward block
            y = nn.LayerNorm()(x)
            y = nn.relu(nn.Dense(h)(y))
            y = nn.Dense(h)(y)
            x = x + y

        query_representation = x[:, -1, :]   # last (query) token
        value_prediction     = nn.Dense(1)(query_representation)  # [B,1]
        return jnp.squeeze(value_prediction, -1)                  # [B]
\end{lstlisting}

\textbf{Note}: We reduce the font size to accomodate the entire block on a single page.

\section{EvA-RL Algorithm}
\label{sec:algo}

\begin{algorithm}[H]
\caption{Policy optimization with EvA-RL}
\label{alg:eva_rl}
\begin{algorithmic}[1]
\REQUIRE Policy $\pi_\theta$, value evaluator $\hat{V}_D^{\pi_\theta, \phi}$, evaluability coefficient $\beta$, warm-up steps $T_\mathrm{warm}$
\STATE Initialize empty replay buffer $\mathcal{B}$
\FOR{$t = 0, 1, 2, \dots$}
    \STATE Roll out $\pi_{\theta_t}$ in $\mathcal{M}_A$ and $\mathcal{M}_D$
    \STATE Store tuple $(s,\, G(H),\, \Xi_A)$ in $\mathcal{B}$; retain data from $m$ most recent policies
    \IF{$t > T_{\mathrm{warm}}$}
        \STATE Update $\phi$: minimize squared error of evaluator predictions for policies in $\mathcal{B}$
        \STATE Update $\theta$: maximize EvA-RL objective (Eq.~\ref{eq:value_predictability_objective}) with updated $\phi$
    \ELSE
        \STATE Update $\theta$: standard policy gradient (no evaluability penalty)
    \ENDIF
\ENDFOR
\end{algorithmic}
\end{algorithm}

\section{Hyperparameter Details}\label{sec:appendix_hypers}

In this section, we provide details of the hyperparameters pertaining to the EvA-RL algorithm. The hyperparameters for policy gradient optimization using advantage-actor critic follow the default values in the PureJaxRL library~\citep{lu2022discovered} for Gymanx and Brax environments.

\subsection{Gymnax MinAtar Environments}

\begin{table}[H]
    \centering
    \begin{tabular}{lll}
        \toprule
        \textbf{Sr. No.} & \textbf{Hyperparameter} & \textbf{Value} \\
        \midrule
        1 & Number of Environments & 64 \\
        2 & Number of Steps per Environment & 100 \\
        3 & Total Timesteps & 1e7 \\
        4 & Assessment Horizon & 10 \\
        5 & General Trajectory Length & 200 \\
        6 & Learning Rate & 1e-4 \\
        7 & Number of Epochs & 100 \\
        8 & Number of Heads in Transformer & 4 \\
        9 & Number of Layers in Transformer & 4 \\
        10 & Hidden Dimension in Transformer & 16 \\
        11 & Number of Assessment Start States & 5 \\
        12 & Batch Size & 256 \\
        13 & Evaluability Coefficient & 0, 1e-2, 1e-1 \\
        14 & Number of Evaluator Updates & 5 \\
        15 & Seed & 0, 1, \dots{}, 19 \\
        16 & Maximum Buffer Size & 32000 Transitions \\
        \bottomrule
    \end{tabular}
\end{table}

\subsection{Brax Robotics Simulations}
\begin{table}[H]
    \centering
    \begin{tabular}{lll}
        \toprule
        \textbf{Sr. No.} & \textbf{Hyperparameter} & \textbf{Value} \\
        \midrule
        1 & Number of Environments & 2048 \\
        2 & Number of Steps per Environment & 10 \\
        3 & Total Timesteps & 2e7 \\
        4 & Assessment Horizon & 25 \\
        5 & General Trajectory Length & 1000 \\
        6 & Batch Size & 256 \\
        7 & Learning Rate & 1e-4 \\
        8 & Number of Epochs & 100 \\
        9 & Number of Heads in Transformer & 4 \\
        10 & Number of Layers in Transformer & 4 \\
        11 & Hidden Dimension in Transformer & 16 \\
        12 & Number of Assessment Start States & 5 \\
        13 & Number of Evaluator Updates & 5 \\
        14 & Evaluability Coefficient & 0, 5e-4 \\
        15 & Seed & 0, 1, \dots{}, 19 \\
        16 & Maximum Buffer Size & 102400 Transitions \\
        \bottomrule
    \end{tabular}
\end{table}

\subsection{Dynamics Mismatch Experiment}
\label{sec:appendix_dynamics_mismatch}
\begin{table}[H]
\centering
\begin{tabular}{clc}
\toprule
\textbf{Sr. No.} & \textbf{Hyperparameter} & \textbf{Value} \\
\midrule
\multicolumn{3}{l}{\textit{Environment}} \\
1  & State space & Continuous, $[0,1]^2$ \\
2  & Action space & Discrete, 4 (cardinal directions) \\
3  & Deployment step size $\delta_D$ & $0.2$ \\
4  & Assessment step size $\delta_A$ & $0.4$ \\
5  & Discount factor $\gamma$ & $0.99$ \\
6  & Max deployment steps per episode & $50$ \\
7  & Goal center & $(0.9, 0.9)$ \\
8  & Goal radius & $0.08$ \\
9  & Deployment start state & $(0.05, 0.05)$ \\
\midrule
\multicolumn{3}{l}{\textit{Training}} \\
10 & Number of parallel environments & $8$ \\
11 & Steps per environment per update & $200$ \\
12 & Total policy update iterations & $1000$ \\
13 & Warm-up iterations & $100$ \\
14 & Policy learning rate & $5 \times 10^{-4}$ \\
15 & Critic learning rate & $1 \times 10^{-3}$ \\
16 & Evaluator learning rate & $3 \times 10^{-3}$ \\
17 & Entropy coefficient & $0.01$ \\
18 & Evaluability coefficient $\beta$ & $0.05$ \\
19 & Number of seeds & $30$ \\
\midrule
\multicolumn{3}{l}{\textit{Assessment environment}} \\
20 & Number of start states $K$ & $4$ \\
21 & Assessment horizon $H$ & $3$ \\
22 & Assessment discount $\gamma_{\text{ae}}$ & $1$ \\
\midrule
\multicolumn{3}{l}{\textit{Replay buffer}} \\
23 & Batch size & $128$ \\
24 & Maximum buffer size & $16{,}000$ transitions \\
25 & Offline episode store size & $1{,}000$ episodes \\
\midrule
\multicolumn{3}{l}{\textit{Value evaluator (transformer)}} \\
26 & Hidden dimension & $32$ \\
27 & Number of attention heads & $2$ \\
28 & Number of layers & $2$ \\
\midrule
\multicolumn{3}{l}{\textit{Policy and critic networks}} \\
29 & Hidden dimension & $64$ \\
30 & Number of hidden layers & $2$ \\
31 & Activation function & Tanh \\
\bottomrule
\end{tabular}
\caption{Hyperparameters for the dynamics mismatch experiment (Section~\ref{sec:mismatch}).}
\label{tab:hypers_mismatch}
\end{table}

\section{Assessment environment state instances}\label{sec:appendix_assessment_env_states}

\begin{figure}[H]
    \centering
    \begin{subfigure}[t]{.32\linewidth}
        \centering
        \includegraphics[width=.7\textwidth]{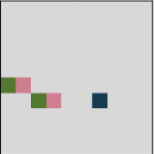}
        \caption{Asterix}
    \end{subfigure}
    \hfill
    \begin{subfigure}[t]{.32\linewidth}
        \centering
        \includegraphics[width=.7\textwidth]{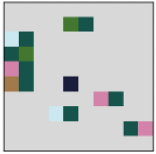}
        \caption{Freeway}
    \end{subfigure}
    \hfill
    \begin{subfigure}[t]{.32\linewidth}
        \centering
        \includegraphics[width=.7\textwidth]{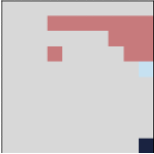}
        \caption{SpaceInvaders}
    \end{subfigure}
    \caption{Assessment environment start-states for Gymnax environments. In Asterix, a collector (black) must gather coins (green/pink) approaching from the left side of the screen. In Freeway, a chicken (black) crosses the freeway vertically while avoiding horizontally moving cars (dual pixels; green indicates the rear end). In SpaceInvaders, a protector (black) faces an alien spaceship (red) that has launched a fireball (white) toward it. These assessments evaluate how an RL policy performs in scenarios it may encounter during training. For instance, the Asterix agent is tested on whether it can capture two coins approaching from the left. The Freeway agent is tested on whether it can advance the chicken past traffic while the road is momentarily clear. The SpaceInvaders agent is tested on how it handles a fireball heading directly toward the protector.}
    \label{fig:dt_minatar}
\end{figure}

\clearpage
\section{Additional Experimental Results}\label{sec:appendix_addl_experimental_results}

\begin{figure}[H]            
    \centering
    \includegraphics[width=\linewidth]{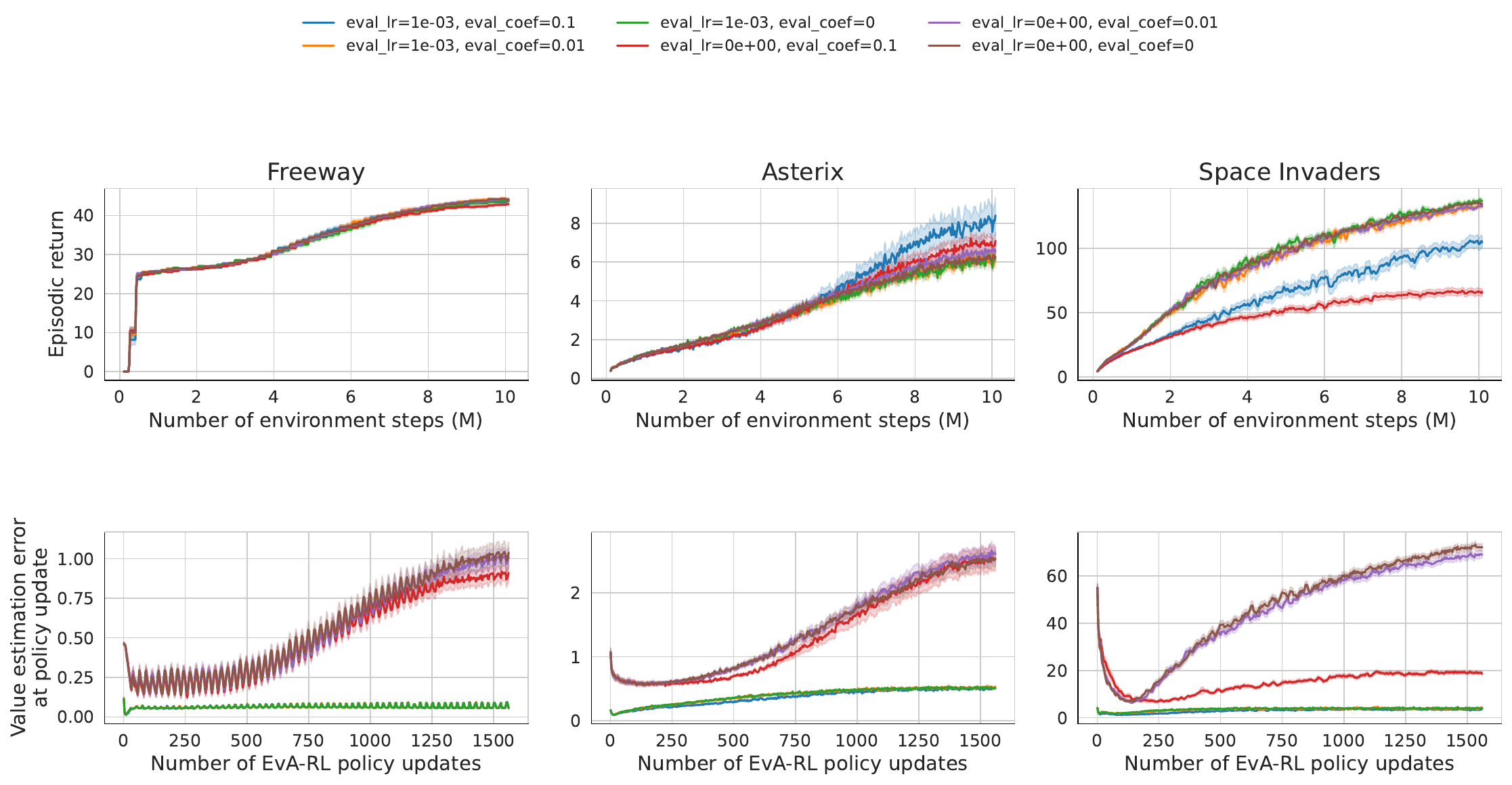}
    \caption{Learning curves for discrete-action environments. Mean and standard error of the episodic return and value evaluation error (computed immediately before each policy update) across 20 seeds for Asterix, Freeway, and SpaceInvaders. Learning curves span 10M environment interactions. In the legend, the evaluability coefficient varies among $\{0,\, 0.01,\, 0.1\}$ and the value evaluator learning rate is either $0$ or $10^{-3}$; a learning rate of $0$ indicates a frozen evaluator pre-trained on data from standard policy gradient RL, and an evaluability coefficient of $0$ corresponds to standard policy gradient RL. When the evaluator is co-learned alongside the policy, returns closely follow those of standard RL while evaluation errors are significantly lower. With a frozen evaluator, returns are lower than in the co-learned setting. The co-learned evaluator's error maintains a near zero value over the course of training, allowing co-learning to provide reliable evaluability feedback to the policy during evaluation-aware learning.}
    \label{fig:minatar_training_curves}
\end{figure}

\begin{figure}[H]            
    \centering
    \includegraphics[width=\linewidth]{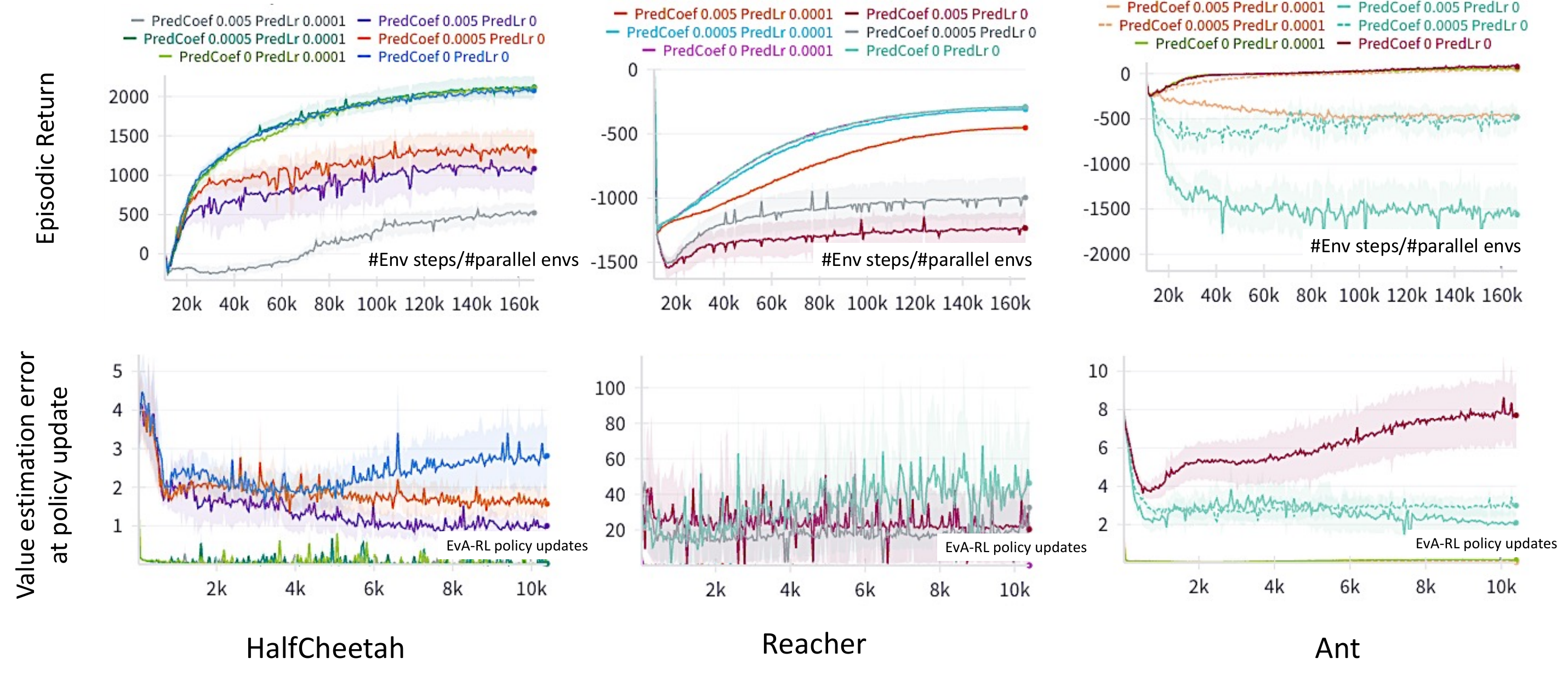}
    \caption{Training curves for continuous-action environments. Mean and standard error of the episodic return and value evaluation error (computed immediately before each policy update) across 5 seeds for the Brax environments of HalfCheetah, Reacher, and Ant (note that the $\beta=5\times 10^{-4}$ with learning rate of $10^{-4}$ numbers reported in the experiment section are still with 20 seeds). Each trial is performed on a single NVIDIA A100 80GB GPU. Learning curves span 10 million environment interactions. In the legend, `PredCoef' and `PredLr' denote the evaluability coefficient and value evaluator learning rate, respectively. The evaluability coefficient varies among $\{0,\, 5 \times 10^{-4},\, 5 \times 10^{-3}\}$ and the evaluator learning rate is either $0$ or $10^{-4}$; a learning rate of $0$ indicates a frozen evaluator pre-trained on data from standard policy gradient RL. The trends are consistent with the discrete-action results in Figure~\ref{fig:minatar_training_curves}, though these environments are more sensitive to changes in the evaluability coefficient.}
    \label{fig:brax_training_curves}
\end{figure}

\section{OPE Baseline Data}\label{sec:appendix_ope_data}

The OPE baselines are provided only with deployment trajectories and not assessment rollouts. Assessment trajectories are collected over short, fixed horizons without discounting ($\gamma_{\text{ae}} = 1$), producing truncated trajectories that are incompatible with importance-sampling-based estimators (TIS, PDIS, DR), which require full-horizon trajectories under a consistent discount factor to compute valid importance ratios. FQE similarly cannot use assessment data, as it fits Q-values via Bellman backups that assume the same discount factor used at deployment.


\end{document}